\documentclass[lettersize,journal]{IEEEtran}
\usepackage{amsmath,amsfonts}
\usepackage{algorithmic}
\usepackage{algorithm}
\usepackage{array}
\usepackage[caption=false,font=normalsize,labelfont=sf,textfont=sf]{subfig}
\usepackage{textcomp}
\usepackage{stfloats}
\usepackage{url}
\usepackage{verbatim}
\usepackage{graphicx}
\usepackage{lineno}
\usepackage{cite}
\usepackage{multirow}
\usepackage{booktabs}  % 对应 \toprule, \midrule, \bottomrule
\usepackage{amsthm}          % 给出 proof 环境

\usepackage{amsmath}  % 提供数学公式支持 (通常都需要)
\usepackage{amsthm}   % 提供定理环境支持 (核心包)
\newtheorem{theorem}{Theorem}

\begin{document}
%1、Momentum-Driven Adaptive Weighting to Boost Heterogeneous Knowledge Distillation 2、Decoupling Space and Time with Momentum-Guided Adaptation for Heterogeneous Knowledge Distillation
%Stabilizing Heterogeneous Knowledge Distillation via Feature Geometry Decoupling and Dynamic Gradient Rectification

\title{Heterogeneous Knowledge Distillation via Geometry Decoupling and Momentum-Aware Gradient Regulation}

\author{Wuming~Yang,
        Xiang~Zhang, 
        and~Hongmin~Zhao* % 如果导师有 IEEE 会员身份，可以加上 ,~\IEEEmembership{Member,~IEEE}
        % <-this % stops a space
\thanks{W. Yang and H. Zhao are with the School of Computer and Mathematics, Central South University of Forestry and Technology, Changsha, Hunan, China (e-mail: yangwuming@csuft.edu.cn; zhm810@csuft.edu.cn).}% <-this % stops a space
\thanks{X. Zhang is with the College of Computer Science and Technology, and the Laboratory of Digitizing Software for Frontier Equipment, National University of Defense Technology, Changsha, Hunan, China (e-mail: zhangxiang08@nudt.edu.cn).}% <-this % stops a space
\thanks{*Corresponding author: Hongmin Zhao.}}

% The paper headers =IEEE Transactions on Neural Networks and Learning Systems
\markboth{}%
{Shell \MakeLowercase{\textit{et al.}}: A Sample Article Using IEEEtran.cls for IEEE Journals}

%\IEEEpubid{0000--0000/00\$00.00~\copyright~2021 IEEE}
% Remember, if you use this, you must call \IEEEpubidadjcol in the second
% column for its text to clear the IEEEpubid mark.

\maketitle

\begin{abstract}
Heterogeneous Knowledge Distillation (HKD) aims to transfer knowledge across varying architectures (e.g., from Transformer to CNN) but inherently suffers from severe training instability. We reveal that this instability stems from two highly coupled challenges: massive feature norm discrepancies that cause optimization drag, and severe gradient conflicts between the primary and distillation objectives arising from distinct inductive biases. To achieve stable distillation, we propose SPOFA, a framework built upon a novel Feature and Gradient Dual Stabilization mechanism. Specifically, at the feature level, we introduce a LayerNorm-based decoupling projector that explicitly decouples feature magnitude from direction, creating a bounded and stable space for semantic alignment. At the gradient level, we propose a momentum-driven Exponential Moving Average (MEMA) dynamic scaler. By establishing a robust historical baseline of the optimization trajectory, MEMA actively evaluates instantaneous gradient conflicts and adaptively penalizes harmful distillation signals, guaranteeing stable convergence. Importantly, SPOFA achieves this dual stabilization with an extremely lightweight parameter footprint. Extensive experiments on two mainstream benchmarks demonstrate that SPOFA achieves state-of-the-art accuracy, significantly outperforming computationally expensive methods while introducing only minimal computational overhead compared to standard baselines.
\end{abstract}


%引言
\section{Introduction}
\label{sec:intro}

\IEEEPARstart{D}{eep} neural networks achieve groundbreaking success but often sacrifice accuracy to work on practical low-resource devices, such as mobile phones and edge computing nodes. Knowledge Distillation (KD) \cite{hinton2015distilling} offers an alternative solution by transferring knowledge from a cumbersome, high-capacity teacher to a lightweight student. Recently, KD has even been demonstrated as a robust solution to resolve structural missingness in highly compressed domains \cite{jing2026reflectance}. However, knowledge transfer becomes particularly challenging in heterogeneous settings (e.g., Transformer teacher to CNN student). Unlike homogeneous distillation, where feature topologies are inherently isomorphic, heterogeneous knowledge transfer suffers from severe training instability. This instability arises from fundamentally distinct architectural priors and inductive biases—for instance, the global self-attention mechanism in Transformers versus the local receptive fields in CNNs.

\begin{figure}[!t]
    \centering
    \includegraphics[width=\columnwidth]{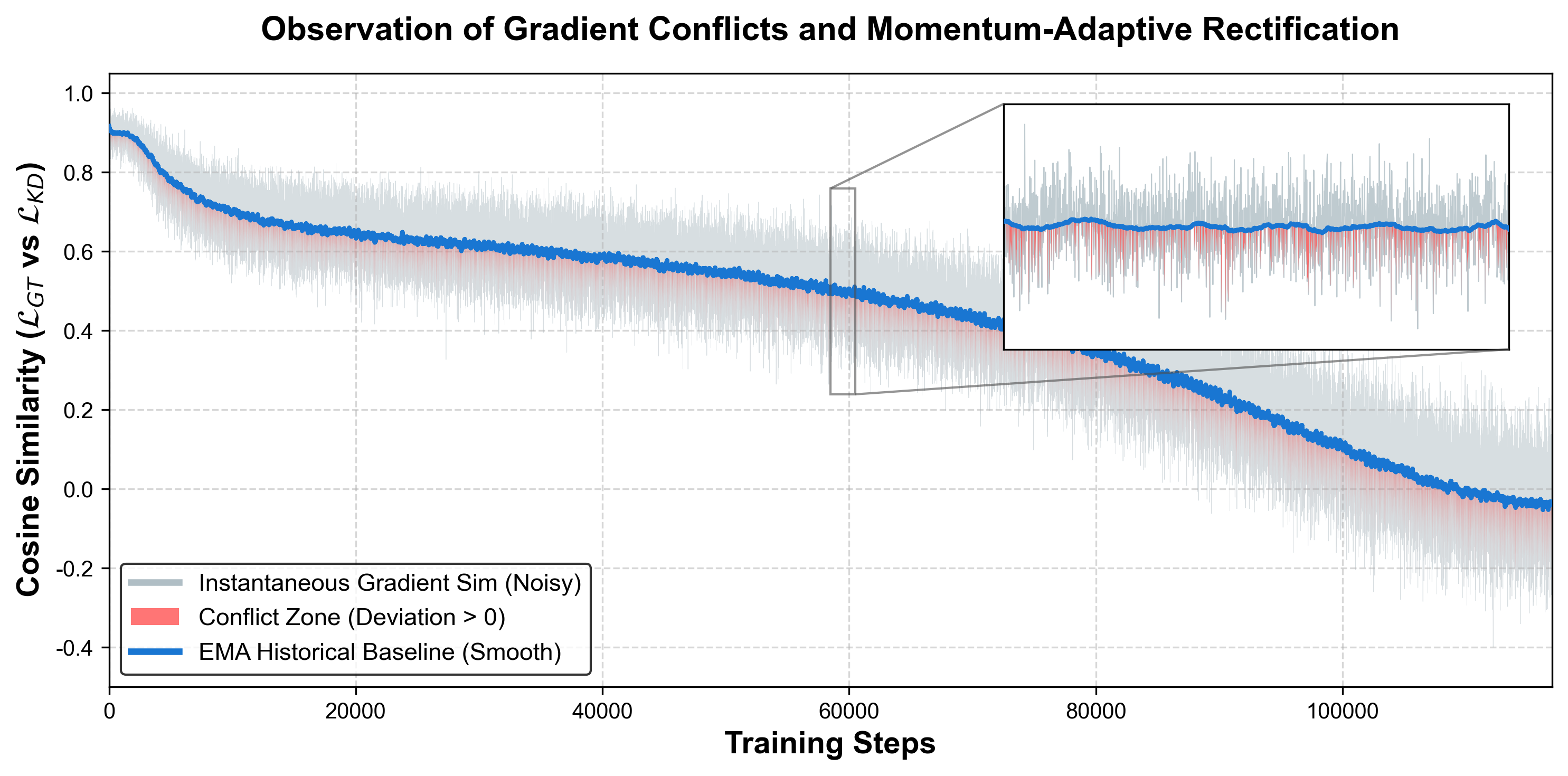}
    \caption{Observation of Gradient Conflicts and the SPOFA Rectification Mechanism. The highly volatile instantaneous cosine similarity (grey line) between the primary task ($\nabla \mathcal{L}_{GT}$) and distillation ($\nabla \mathcal{L}_{KD}$) gradients highlights the severe optimization noise inherent in heterogeneous knowledge distillation. Our proposed SPOFA framework introduces a Momentum-driven EMA baseline (blue line) to track the historical optimization trajectory. When severe gradient conflicts are detected—specifically, when the instantaneous similarity drops below the historical baseline (red zone)—SPOFA explicitly penalizes the conflicting distillation signals, thereby suppressing harmful distillation signals and encouraging a smoother optimization trajectory. The visualization is based on the Swin-T (Teacher) and ResNet-18 (Student) pair.}
    \label{fig:teaser_gradient_conflict}
\end{figure}

While existing approaches like OFA \cite{ofa2023neurips} and PAT \cite{lin2025perspective} attempt to bridge this representational gap via linear projection or heavy attention modules, they overlook the root causes of this training instability.  To investigate these fundamental issues, we conduct an empirical analysis on a highly representative heterogeneous pair: Swin-T (Teacher) to ResNet-18 (Student). This pair encapsulates a severe architectural clash between the global self-attention mechanism of Transformers and the local receptive fields of CNNs. By dissecting this specific combination, we reveal that the instability in Heterogeneous Knowledge Distillation (HKD) is fundamentally driven by two highly coupled challenges:  First, at the feature level (Structural Gap), architectures with distinct inductive biases exhibit substantial discrepancies in feature magnitudes (L2 Norms) \cite{raghu2021vision}. Specifically, the global receptive field of Swin-T intrinsically generates unbounded and massive feature magnitudes, whereas ResNet-18 expects normalized, localized variance. As illustrated in our preliminary analysis (detailed later in Section III-B), forcing the CNN student to mimic these unnormalized representations causes a severe \textit{representational bottleneck}. The network squanders its limited capacity on fitting absolute scalar norms rather than aligning the actual semantic directions. Second, at the gradient level (Optimization Gap), optimizing across these structural gaps inherently induces severe gradient conflicts. Because the CNN and Transformer dictate fundamentally different descent directions, the auxiliary distillation gradients ($\nabla \mathcal{L}_{KD}$) frequently oppose the primary task gradients ($\nabla \mathcal{L}_{GT}$) under heterogeneous noise (as explicitly depicted in Fig.~\ref{fig:teaser_gradient_conflict}). 
Existing methods typically treat this highly non-stationary process with static or memoryless dynamic weights, which exacerbates \textit{optimization instability} and—akin to the error amplification recently observed in uncalibrated collaborative learning \cite{zhou2026trust}—triggers catastrophic forgetting of the primary classification knowledge.

To mitigate these two coupled challenges, we propose SPOFA(Spatial Projector and Momentum-based Adaptive One-For-All), a robust framework built upon a novel Feature and Gradient Dual Stabilization paradigm. Rather than treating the symptoms with heavy attention modules or memoryless heuristics, SPOFA explicitly maps targeted interventions to the identified root causes: Specifically, to overcome the structural gap at the feature level, we introduce a minimalist LayerNorm-based decoupling projector. By explicitly isolating feature magnitude from direction optimization, this projector bounds the semantic space and alleviates the aforementioned representational bottlenecks.  Simultaneously, to resolve the optimization gap at the gradient level, we propose a Momentum-driven Exponential Moving Average (MEMA) dynamic scaler (as illustrated in Fig.~\ref{fig:teaser_gradient_conflict}). Instead of relying on volatile instantaneous metrics, MEMA establishes a robust historical baseline of the optimization trajectory. It actively probes gradient conflicts and adaptively penalizes harmful distillation signals, thereby improving optimization stability during heterogeneous distillation.

Our main contributions are summarized as follows:
\begin{itemize}
    \item We analyze heterogeneous knowledge distillation from a stability-oriented perspective and identify two coupled factors that are strongly associated with unstable training: feature-level norm discrepancy and optimization-level gradient conflict.
    \item We propose a feature geometry decoupling strategy implemented by a lightweight LayerNorm-based auxiliary projector. It reduces architecture-dependent magnitude interference before intermediate-logit distillation, encouraging the student to focus more on semantic directional alignment.
    \item We develop a Momentum-driven EMA (MEMA) conflict regulation mechanism. By maintaining a historical gradient-consistency baseline, MEMA adaptively suppresses harmful distillation signals when abnormal sample-wise conflicts are detected in the output logit space.
    \item Extensive experiments show that SPOFA mostly outperforms the strong OFA baseline and achieves state-of-the-art accuracy against computationally expensive methods (e.g., PAT) on ImageNet and CIFAR, while introducing only minimal computational overhead.
\end{itemize}

%相关工作部分内容
\section{Related Work}

\subsection{Heterogeneous Knowledge Distillation}
Homogeneous Knowledge Distillation typically assumes that the teacher and student models share isomorphic architectures (e.g., ResNet to ResNet) \cite{hinton2015distilling,zagoruyko2016paying}, allowing for straightforward feature alignment. However, modern deployment often requires compressing massive Transformer-based models into lightweight CNNs. Heterogeneous Knowledge Distillation (HKD) aims to bridge this severe architectural gap.

Early approaches focused on intermediate feature mimicry via basic projection layers \cite{romero2014fitnets} or relational topology \cite{park2019relational}. The One-for-All (OFA) framework \cite{ofa2023neurips} significantly advanced HKD by projecting intermediate features into a unified, architecture-agnostic logit space. Most recently, Perspective-Aware Teaching (PAT) \cite{lin2025perspective} introduced Region-Aware Attention (RAA) and Adaptive Feedback Prompts (AFP) to re-blend features across different inductive biases. Furthermore, the necessity of heterogeneous alignment has been echoed in complex paradigms such as masked image modeling \cite{wang2023heterogeneous}, data-free federated learning \cite{zhu2021data}, wearable human activity recognition \cite{xiao2025heterogeneous}, and asymmetric visual representation retrieval \cite{xie2024d3still}.

Despite improving semantic alignment, these methods encounter two critical bottlenecks. First, methods like PAT rely on heavy attention modules, resulting in prohibitive computational overhead (e.g., a $3\times$ increase in training latency compared to OFA \cite{lin2025perspective}). Second, and more importantly, they overlook the inherent \textit{training instability} of HKD. By ignoring the immense norm discrepancies \cite{raghu2021vision} and the volatile gradient conflicts across heterogeneous topologies, existing methods force the student into a suboptimal and chaotic optimization trajectory. SPOFA addresses this fundamental flaw through a lightweight, dual-stabilization paradigm.

\subsection{Feature-Level Alignment and Norm Decoupling}
Feature normalization techniques, such as BatchNorm and LayerNorm, are foundational to stabilizing deep neural network training by smoothing the optimization landscape. In the context of KD, bridging the capacity gap requires proper feature transformation. While some advanced methods explore complex feature whitening or correlation alignment \cite{ahn2024style} to decouple feature statistics, leverage decoupled contrastive objectives for cross-domain purification \cite{li2025decoupled}, or explicitly separate uncorrelated components in orthogonal spaces \cite{wang2025orthogonal, miles2024vkd} to avoid transferring ambiguous, coupled semantics \cite{wei2024scaled}, they incur massive computational costs in computing complex metrics. Consequently, mainstream HKD frameworks still predominantly rely on linear projectors or complex attention mechanisms. Crucially, these existing approaches merely perform feature transformation rather than true structural decoupling. As recently highlighted in heterogeneous knowledge transfers from CNNs to Vision Transformers \cite{chen2026distilling}, effectively bridging the representational gap requires explicitly addressing the underlying feature structures. For instance, OFA's linear projection and PAT's spatial blending \cite{lin2025perspective} still operate on magnitude-coupled features. When transferring knowledge between highly heterogeneous topologies, forcing the student to mimic absolute, unnormalized feature representations triggers severe representational bottlenecks—the student expends significant capacity trying to match the teacher's arbitrary feature magnitudes rather than learning its semantic directions. To this end, SPOFA introduces a minimalist LayerNorm-based decoupling projector that explicitly isolates magnitude from directional alignment, creating a bounded and stable space to eliminate feature-level optimization inefficiencies with near-zero overhead.

\subsection{Gradient-Level Stabilization and Momentum Dynamics}
Beyond spatial feature alignment, Knowledge Distillation is fundamentally a multi-objective optimization problem, requiring a delicate balance between the primary task gradient ($\nabla \mathcal{L}_{GT}$) and the distillation gradient ($\nabla \mathcal{L}_{KD}$). In Multi-Task Learning (MTL), gradient conflicts are actively addressed through projection methods like PCGrad \cite{yu2020gradient} and CAGrad \cite{liu2021conflict}. Recently, the community has increasingly recognized that similar gradient conflicts severely bottleneck Heterogeneous Knowledge Distillation. For instance, recent state-of-the-art studies, such as DTO-KD \cite{hayderdto}, explicitly identify that architectural mismatches inherently induce severe gradient conflicts between primary and distillation objectives. However, to resolve these conflicts, such methods typically rely on directly applying MTL projection algorithms, which require computing gradient similarities across the entire high-dimensional parameter space, rendering them computationally prohibitive for typical KD pipelines.

Alternatively, researchers have attempted to mitigate interference via dynamic weighting, heuristic loss scaling, adaptive teacher strategies that handle significant prediction discrepancies \cite{huang2025dist+}, or explicitly manipulating policy gradients to overcome misaligned objectives \cite{gu2025safe}. For instance, recent advances in robust model compression have demonstrated that timely knowledge correction integrated with dynamic weight adjustment can significantly stabilize learning trajectories under noisy or mismatched supervision \cite{lan2025continuous}. Furthermore, Momentum and Exponential Moving Average (EMA) techniques have been widely successful in self-supervised learning (e.g., MoCo) to provide stable optimization targets.

Despite these advances, existing dynamic weighting mechanisms in HKD rely almost exclusively on memoryless instantaneous heuristics. In heterogeneous pairs, because the student is constantly forced to map its representations onto an architecturally alien manifold, the batch-level gradients are inherently plagued by high variance and severe noise. Consequently, memoryless metrics frequently trigger chaotic weighting adjustments, exacerbating gradient conflicts and causing optimization instability. SPOFA posits that leveraging a historical baseline is key to stabilizing this chaotic process. By designing the MEMA dynamic scaler, we establish a robust optimization trajectory, adaptively penalizing $\nabla \mathcal{L}_{KD}$ only when instantaneous gradient conflicts diverge from the historical expectation, thereby encouraging more stable optimization dynamics during heterogeneous distillation.

%方法部分内容
\section{Method}

In this section, we present SPOFA, a robust framework designed to resolve the training instability inherent in heterogeneous Knowledge Distillation (HKD). As illustrated in Fig.~\ref{fig:overall_framework}, our framework tackles the heterogeneous gap from two orthogonal dimensions: Feature-Level Stabilization and Gradient-Level Stabilization. 
We first revisit the OFA baseline to formulate the distillation objective and identify its inherent flaws (Sec. \ref{sec:preliminaries}). We then introduce our LayerNorm-based structural decoupling mechanism to eliminate feature-level representational bottlenecks (Sec. \ref{sec:feature_level}). Next, we formulate the EMA-Guided Dynamic Scaler (MEMA) with a novel per-sample gradient probing mechanism to rectify optimization trajectories at the gradient level, supported by a sufficient-condition analysis (Sec. \ref{sec:gradient_level}). Finally, we summarize the overall dual-stabilization framework and training algorithm (Sec. \ref{sec:overall}).

% --- 插入主框架大图 ---
\begin{figure*}[t]
    \centering
    \includegraphics[width=\textwidth]{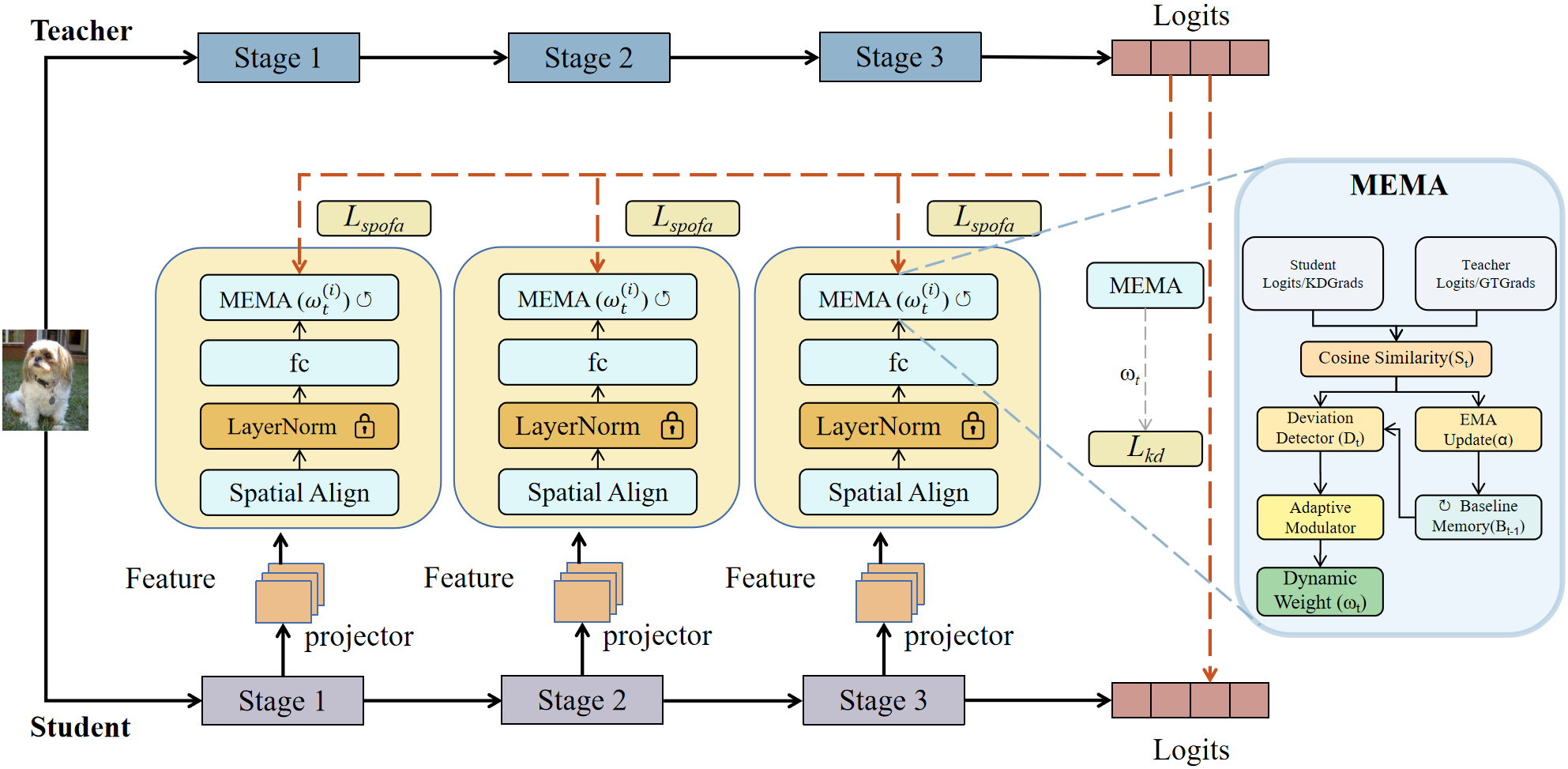} 
    \caption{The overall architecture of the proposed SPOFA framework. \textbf{Left:} The multi-level heterogeneous distillation pipeline from a teacher to a student. \textbf{Middle :} The intermediate distillation blocks, each integrating a LayerNorm-based decoupling projector (comprising \textit{Spatial Align}, LayerNorm, and \textit{fc} layers) to reduce architecture-dependent magnitude interference and alleviate representational bottlenecks, accompanied by a Stage MEMA ($\omega_t^{(i)}$) for dynamic local loss scaling. \textbf{Right :} The inner workings of the Momentum-driven EMA (MEMA) scaler serve as a universal dynamic module. When applied to intermediate layers, it targets projected features with both reward and penalty mechanisms; when applied to the output layer, it targets gradient vectors and exclusively employs a penalty mechanism via one-sided truncation to regulate harmful distillation signals under gradient conflicts. Note: Only three stages are shown for convenience. In our experiments, all models are split into four stages.}
    \label{fig:overall_framework}
\end{figure*}
% ---------------------

\subsection{Preliminaries: The OFA Baseline and its Flaws}
\label{sec:preliminaries}
Consider a pre-trained teacher network $T$ and a student network $S$ with heterogeneous architectures (e.g., Transformer vs. CNN). Let $X \in \mathbb{R}^{B \times C \times H \times W}$ denote the input data and $Y$ denote the corresponding ground-truth labels. The intermediate features extracted from the $i$-th selected stage of the teacher and student are denoted as $F_i^T \in \mathbb{R}^{C_i^T \times H_i^T \times W_i^T}$ and $F_i^S \in \mathbb{R}^{C_i^S \times H_i^S \times W_i^S}$, respectively. 

Due to the fundamental architectural gap, directly matching $F_i^T$ and $F_i^S$ is infeasible. The One-for-All (OFA) framework \cite{ofa2023neurips} tackles this by introducing an auxiliary projector $P_{OFA}(\cdot)$ that maps the intermediate features of the student into a unified, architecture-agnostic logits space:
\begin{equation} \label{eq:projector}
    Z_i^S = P_{OFA}(F_i^S) = W_{cls}^{(i)} \cdot \phi(F_i^S)
\end{equation}
where $\phi(\cdot)$ denotes a lightweight spatial alignment operation and $W_{cls}^{(i)}$ is the classifier weight. 

Instead of using standard KL-divergence, OFA introduces an adaptive target enhancement mechanism. Let $p^S = \sigma(Z_i^S / \tau)$ and $p^T = \sigma(Z^T / \tau)$ be the softened probabilities scaled by temperature $\tau$, and let $y$ be the one-hot ground-truth label. The intermediate distillation loss $\mathcal{L}_{OFA}^{(i)}$ is formulated as:
\begin{equation} \label{eq:ofa_loss}
    \mathcal{L}_{OFA}^{(i)} = \sum_{c=1}^{C} - \left( (p_{c}^T + y_{c})^{\epsilon} - y_{c} \right) \log(p_{c}^S)
\end{equation}
where $\epsilon \geq 1$ is a modulating parameter.

While OFA successfully bypasses dimensional mismatch, it exposes two severe underlying flaws that induce training instability: (1) \textit{Feature-Level Flaw:} Linear projection and basic activations preserve absolute feature magnitudes, failing to address the massive norm discrepancy between heterogeneous architectures. (2) \textit{Gradient-Level Flaw:} The constant or memoryless stage-wise weighting of $\mathcal{L}_{OFA}$ ignores the highly volatile gradient conflicts between the primary task and the distillation process, inevitably leading to optimization noise and trajectory oscillations.

\subsection{Feature-Level Alignment: Breaking Representational Bottlenecks via Norm Decoupling}
\label{sec:feature_level}

In heterogeneous distillation, models with distinct inductive biases operate on vastly different feature scales. As demonstrated in Fig.~\ref{fig:feature_analysis} (Left), directly forcing a CNN student to mimic the feature space of a ViT teacher results in severe Norm Discrepancy. Standard projectors entangle the optimization of feature magnitude (L2 Norm) and feature direction (Cosine Similarity). Consequently, the student squanders its limited representational capacity, struggling to scale its absolute feature values, leading to severe representational bottlenecks.

\begin{figure}[t]
    \centering
    \includegraphics[width=\linewidth]{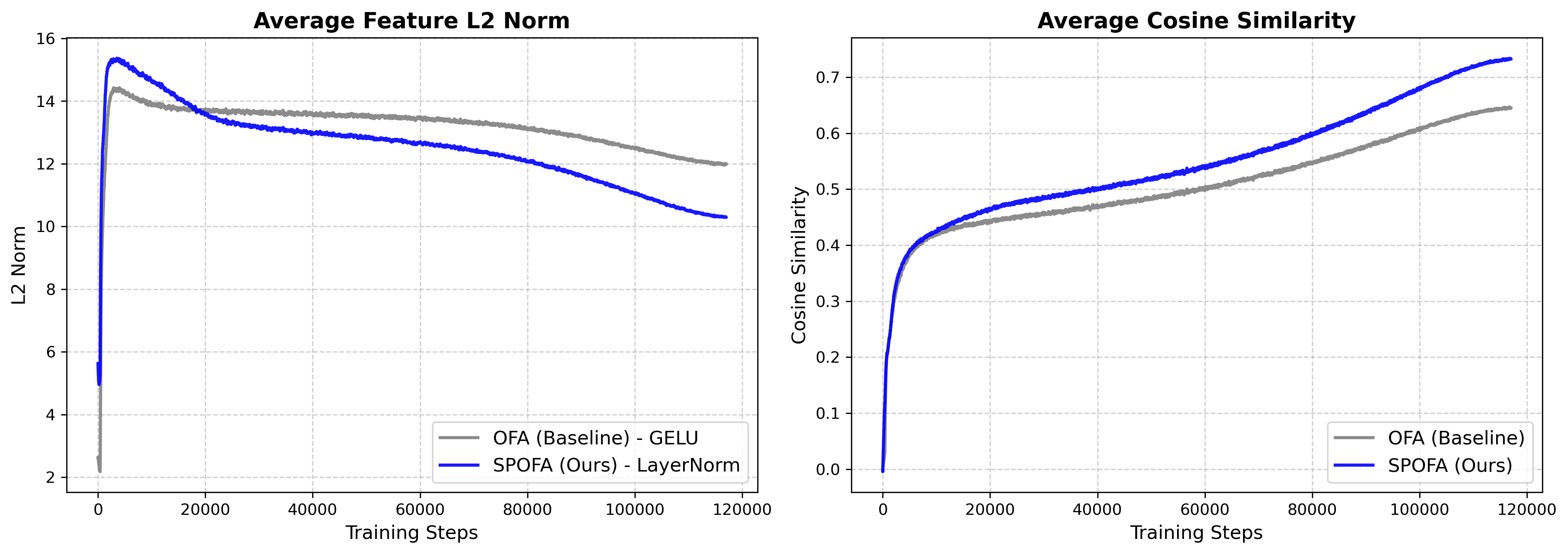}
    \caption{Comparative analysis of feature-level dynamics during heterogeneous distillation. The visualization is based on the \textbf{Swin-T (Teacher)} and \textbf{ResNet18 (Student)} pair. By replacing standard linear projectors with a LayerNorm-based structural decoupling mechanism, SPOFA reduces the immense feature magnitude (Left) and alleviates representational bottlenecks, thereby achieving a significantly higher semantic alignment (Right) compared to the OFA baseline.}
    \label{fig:feature_analysis}
\end{figure}

To structurally decouple magnitude from direction—a principle recently highlighted as crucial for regularizing generalizable knowledge transfer \cite{yang2025toward}—SPOFA introduces a minimalist yet highly effective LayerNorm-based Projector, denoted as $P_{SPOFA}(\cdot)$. We embed Layer Normalization (LN) deeply into the auxiliary branch before the final classification projection:

\begin{equation} \label{eq:layernorm}
    \hat{F}_i^S = \text{LayerNorm}(\phi(F_i^S)) = \frac{\phi(F_i^S) - \mu}{\sqrt{\sigma^2 + \epsilon_{LN}}} \odot \gamma + \beta
\end{equation}

where $\phi(\cdot)$ denotes a spatial pooling operation (e.g., global average pooling) that aggregates the spatial dimensions to yield a 1D channel-wise feature vector in $\mathbb{R}^C$; $\mu$ and $\sigma^2$ are the mean and variance computed across the channel dimension, $\epsilon_{LN}$ is a small constant for numerical stability, and $\gamma, \beta$ are learnable affine parameters. It is worth noting that LN is only inserted into the auxiliary distillation branch and is discarded after training together with the projector. Therefore, it does not alter the deployed student architecture. Its role is to reduce architecture-dependent magnitude bias in the distillation space rather than to normalize the student backbone itself.

The SPOFA auxiliary logits are computed as $Z_i^{SPOFA} = W_{cls}^{(i)} \cdot \hat{F}_i^S$. By applying LayerNorm, the feature magnitude is explicitly decoupled from the student's architectural topology. Since the teacher is frozen with constant feature magnitudes, the student only needs to align its relative logit distribution, not the absolute scale. Crucially, when paired with the temperature-scaled Softmax in the distillation loss---which inherently relies on relative probability distributions rather than absolute logit magnitudes---this single-sided structural decoupling is sufficient to forcefully cut off the gradient pathways responsible for arbitrary magnitude scaling. As a result, the optimizer is strongly encouraged to shift its focus towards maximizing the angular consistency between the student and teacher representations. This successfully translates into a significantly higher semantic alignment, as directly validated by the Cosine Similarity improvements in Fig.~\ref{fig:feature_analysis} (Right).

\subsection{Gradient-Level Stabilization: Rectifying Trajectories via MEMA}
\label{sec:gradient_level}
While the feature-level stabilization alleviates representational bottlenecks, gradient-level conflicts persist and must be addressed separately. Even with perfectly aligned features, Knowledge Distillation remains a multi-objective optimization problem. The gradients derived from the ground-truth loss ($\nabla_{\theta} \mathcal{L}_{GT}$) and the distillation loss ($\nabla_{\theta} \mathcal{L}_{KD}$) frequently point in opposing directions. As highlighted by recent advancements in top-tier vision conferences, such unmitigated gradient conflicts can severely trap parameter updates in sub-optimal solutions \cite{patel2025learning}, necessitating dynamically balanced loss weighting \cite{sun2025two}. Previous dynamic weighting methods rely on memoryless instantaneous heuristics. However, due to the high variance of heterogeneous mini-batches, such memoryless approaches trigger chaotic weighting oscillations, exacerbating optimization instability.

To resolve this, SPOFA abandons memoryless heuristics and proposes a Momentum-driven EMA (MEMA) Scaler featuring a fine-grained \textit{per-sample} probing mechanism.

First, at training step $t$, to rigorously monitor optimization dynamics while maintaining a strictly lightweight footprint, we probe gradient conflicts in the \textbf{output logit space}. Inspired by recent findings that sample-wise logit constraints effectively preserve structural semantics \cite{zhu2025ckd}, let $Z^S$ denote the logit bottleneck. We compute the instantaneous Gradient Similarity $S_{t, j}$ specifically for the $j$-th sample in a mini-batch of size $N$:
\begin{equation} \label{eq:sim}
    S_{t, j} = \frac{\langle \nabla_{Z^S} \mathcal{L}_{GT}^{(t, j)}, \nabla_{Z^S} \mathcal{L}_{KD\_Logit}^{(t, j)} \rangle}{\| \nabla_{Z^S} \mathcal{L}_{GT}^{(t, j)} \|_2 \cdot \| \nabla_{Z^S} \mathcal{L}_{KD\_Logit}^{(t, j)} \|_2}
\end{equation}
By computing per-sample gradients in the low-dimensional output logit space rather than across the entire parameter manifold ($\theta$), SPOFA avoids expensive full-parameter gradient comparison while retaining a lightweight conflict indicator. Although logit-level gradient similarity is not equivalent to parameter-level gradient similarity, it provides a computationally tractable proxy for detecting potential conflicts between the primary and distillation objectives. According to the chain rule, the corresponding parameter-level gradients are induced by the logit-level gradients through the shared network Jacobian $\partial Z^S / \partial \theta$. Thus, a strong directional discrepancy in the logit space can indicate potentially harmful optimization interference and can be used to guide adaptive distillation weighting.

\textit{Theoretical Motivation:} Relying solely on instantaneous similarity $S_{t, j}$ for loss scaling is inherently susceptible to mini-batch noise. Inspired by momentum-based stabilization in deep learning optimizers, we establish a robust historical baseline. 
Instead of directly using $S_{t, j}$, we compute a historical baseline $B_t$ of the optimization trajectory using an Exponential Moving Average (EMA) over the batch-averaged similarity:
\begin{equation} \label{eq:ema}
    B_t = \alpha \cdot B_{t-1} + (1 - \alpha) \cdot \frac{1}{N}\sum_{j=1}^{N} S_{t, j}
\end{equation}
where $\alpha \in [0, 1)$ is the momentum coefficient. A higher $\alpha$ dictates a deeper memory, shielding the distillation process from being derailed by transient sample-level outliers.

To guarantee stability, SPOFA adopts a \textit{punish-only} mechanism. We utilize the baseline from the previous step $B_{t-1}$ as the prior expectation. The Deviation $D_{t, j}$ is calculated by comparing the instantaneous similarity against this historical expectation:
\begin{equation} \label{eq:dev}
    D_{t, j} = \max(0, B_{t-1} - S_{t, j})
\end{equation}
This proactive feedback loop structurally distinguishes between "expected architectural conflict" (captured by $B_{t-1}$) and "abnormal gradient interference" (the instantaneous drop). Only when the current conflict is more severe than the historical consensus ($S_{t,j} < B_{t-1}$) does the mechanism trigger a weight decay.

Finally, the dynamic distillation weight $\omega_{t, j}^{out}$ is computed using a bounded linear decay:
\begin{equation} \label{eq:weight}
    \omega_{t, j}^{out} = \operatorname{Clamp}\left( 1.0 - \kappa \cdot D_{t, j}, \; \omega_{min}, \; \omega_{max} \right)
\end{equation}
where $\kappa$ governs the decay sensitivity. When an abnormal conflict is detected ($D_{t, j} > 0$), $\omega_{t, j}^{out}$ decays proportionally. Conversely, expected gradient directions ($D_{t, j} = 0$) retain the base weight of $1.0$.

\textbf{Theoretical Motivation of MEMA.} To provide intuition for why suppressing the distillation contribution under severe gradient conflict can improve primary-task consistency, we analyze a sufficient condition under which the joint update remains consistent with the primary task gradient.

\begin{theorem}[A Sufficient Condition for Primary-Task Gradient Consistency]\label{thm:descent}
Let $g_{gt}^{(t,j)} = \nabla_{Z^S} \mathcal{L}_{GT}^{(t,j)}$ and $g_{kd}^{(t,j)} = \nabla_{Z^S} \mathcal{L}_{KD\_Logit}^{(t,j)}$. Assuming the primary task gradient is non-zero (i.e., $\|g_{gt}^{(t,j)}\|_2 > 0$), the joint optimization step evaluated by $g_{total}^{(t,j)} = g_{gt}^{(t,j)} + \omega_{t,j}^{out} \cdot g_{kd}^{(t,j)}$ forms a strict descent direction for the primary task (i.e., $\langle g_{total}^{(t,j)}, g_{gt}^{(t,j)} \rangle > 0$) if the dynamic weight $\omega_{t,j}^{out}$ satisfies:
\begin{equation}\label{eq:omega_bound}
    \omega_{t,j}^{out} < \frac{\|g_{gt}^{(t,j)}\|_2}{|S_{t,j}| \cdot \|g_{kd}^{(t,j)}\|_2}, \quad \text{when } S_{t,j} < 0.
\end{equation}
When $S_{t,j} \geq 0$, the descent condition inherently holds for any $\omega_{t,j}^{out} \geq 0$.
\end{theorem}

\noindent\textit{Proof.} To strictly ensure that the auxiliary distillation process does not degrade the primary classification task, the inner product must be positive:
\begin{equation}
    \langle g_{gt}^{(t,j)} + \omega_{t,j}^{out} \cdot g_{kd}^{(t,j)}, g_{gt}^{(t,j)} \rangle = \|g_{gt}^{(t,j)}\|_2^2 + \omega_{t,j}^{out} \langle g_{kd}^{(t,j)}, g_{gt}^{(t,j)} \rangle > 0
\end{equation}
Using the definition from Eq.~\ref{eq:sim}, we have $\langle g_{kd}^{(t,j)}, g_{gt}^{(t,j)} \rangle = S_{t,j} \|g_{gt}^{(t,j)}\|_2 \|g_{kd}^{(t,j)}\|_2$. Dividing the inequality by $\|g_{gt}^{(t,j)}\|_2 > 0$ yields:
\begin{equation}
\label{eq:descent_condition}
    \|g_{gt}^{(t,j)}\|_2 + \omega_{t,j}^{out} \cdot S_{t,j} \cdot \|g_{kd}^{(t,j)}\|_2 > 0
\end{equation}
When $S_{t,j} < 0$ (heterogeneous gradient conflict), solving for $\omega_{t,j}^{out}$ establishes a strict mathematically bounded upper limit:
\begin{equation}
\label{eq:omega_bound_proof}
    \omega_{t,j}^{out} < \frac{\|g_{gt}^{(t,j)}\|_2}{|S_{t,j}| \cdot \|g_{kd}^{(t,j)}\|_2}
\end{equation}
\hfill $\square$

\vspace{0.5em}
\noindent{Remark 1 (Bridging Theorem 1 and MEMA).} \label{rmk:bridge}
In the standard OFA framework with static weighting ($\omega \equiv 1$), severe heterogeneous conflicts (large $|S_{t,j}|$) inevitably violate the upper bound in Eq.~\ref{eq:omega_bound}. While stochastic mini-batch training prevents immediate collapse, these continuous bound violations inject severe optimization noise, trapping the student in sub-optimal local minima. 

Directly calculating and applying the exact theoretical bound $\frac{\|g_{gt}^{(t,j)}\|_2}{|S_{t,j}| \cdot \|g_{kd}^{(t,j)}\|_2}$ is practically infeasible due to extreme mini-batch volatility. MEMA provides a robust, history-aware approximation to satisfy this theorem. As the conflict deepens (i.e., $S_{t,j}$ becomes highly negative), the deviation $D_{t, j} = B_{t-1} - S_{t, j}$ proportionally increases. According to Eq.~\ref{eq:weight}, this immediately enforces a severe penalty on $\omega_{t,j}^{out}$, smoothly compressing the distillation influence. By dynamically throttling $\omega_{t,j}^{out}$ inversely proportional to the conflict severity, MEMA acts as a structural safeguard that strongly encourages compliance with the theoretical bound, thereby facilitating a stable and superior descent trajectory, as validated empirically in our subsequent experiments.

\noindent{Remark 2 (Computational Efficiency of Logit-Level Probing).} \label{rmk:complexity}
Calculating per-sample gradient similarities across the entire parameter space ($\theta \in \mathbb{R}^d$) would incur a prohibitive computational complexity of $\mathcal{O}(N \cdot d)$ per step. By structurally evaluating the gradients strictly at the logit bottleneck ($Z^S \in \mathbb{R}^C$), SPOFA explicitly reduces this complexity to $\mathcal{O}(N \cdot C)$. Because the number of categories $C$ is orders of magnitude smaller than the total number of parameters $d$ (e.g., $C \ll d$), the relative computational overhead is strictly bounded by $\mathcal{O}(C/d)$. Consequently, the proposed gradient probing mechanism introduces negligible additional cost, seamlessly integrating into the standard training pipeline.

\subsection{The Dual-Stabilization SPOFA Framework}
\label{sec:overall}
To comprehensively guide heterogeneous distillation, SPOFA explicitly formulates dynamic distillation objectives at two distinct levels, perfectly bridging the feature-level and gradient-level flaws identified previously—a dual-level philosophy that has recently shown great promise in mitigating architectural reliance in other distillation domains \cite{wu2024teacher}.

First, at the output level, the gradient-stabilized distillation loss utilizes the per-sample MEMA weights to rectify optimization trajectories dynamically:
\begin{equation} \label{eq:kd_out}
    \mathcal{L}_{Out}^{(t)} = \frac{1}{N} \sum_{j=1}^{N} \left( \omega_{t, j}^{out} \cdot \mathcal{L}_{KD\_Logit}^{(t, j)} \right)
\end{equation}
where $N$ denotes the mini-batch size, $\omega_{t, j}^{out}$ is the dynamic distillation weight for the $j$-th sample derived from the MEMA scaler (as defined in Eq.~\ref{eq:weight}), and $\mathcal{L}_{KD\_Logit}^{(t, j)}$ represents the standard output logit distillation loss for that specific sample at training step $t$.

Simultaneously, to synchronize multi-layer representation learning, we maintain $K$ independent EMA scalers for the intermediate auxiliary branches. Unlike the output logit space which probes gradient conflicts, the intermediate dynamic weights $\omega_t^{(i)}$ are evaluated based on the batch-averaged feature cosine similarities at stage $i$. The intermediate-level (feature-stabilized) distillation loss is aggregated as:
\begin{equation} \label{eq:kd_inter}
    \mathcal{L}_{Inter}^{(t)} = \sum_{i=1}^{K} \left( \omega_t^{(i)} \cdot \mathcal{L}_{SPOFA}^{(i, t)} \right)
\end{equation}
where $K$ is the total number of intermediate stages, $\omega_t^{(i)}$ serves as the dynamic feature-level weight for the $i$-th stage, and $\mathcal{L}_{SPOFA}^{(i, t)}$ denotes the LayerNorm-decoupled intermediate distillation loss (as defined in Sec.~\ref{sec:feature_level}).

Finally, the overall objective function of SPOFA at training step $t$ explicitly unifies the primary ground-truth classification task with the dual-stabilized distillation components:
\begin{equation} \label{eq:total_loss}
    \mathcal{L}_{Total}^{(t)} = \lambda_{gt} \mathcal{L}_{GT}^{(t)} + \lambda_{out} \mathcal{L}_{Out}^{(t)} + \lambda_{inter} \mathcal{L}_{Inter}^{(t)}
\end{equation}
where $\mathcal{L}_{GT}^{(t)}$ is the primary task loss (e.g., Cross-Entropy) evaluated on the ground-truth labels. The terms $\lambda_{gt}$, $\lambda_{out}$, and $\lambda_{inter}$ are global static hyperparameters introduced to balance the baseline magnitudes of the respective optimization targets. The complete training progression is summarized in Algorithm \ref{alg:spofa}.

\begin{algorithm}[htbp]
\caption{Training Pipeline of the SPOFA Framework}
\label{alg:spofa}
Input: Teacher $T$, Student $S$, Dataset $\mathcal{D}$, Momentum $\alpha$, Sensitivity $\kappa$, Global weights $\lambda_{gt}, \lambda_{out}, \lambda_{inter}$, Learning rate $\eta$. \\
Initialize: Student parameters $\theta$, EMA Baselines $B_0, \{B_0^{(i)}\}_{i=1}^K$ using first-step metrics.
\begin{algorithmic}
\FOR{training step $t = 1, \dots, T_{max}$}
    \STATE Forward pass mini-batch $(x, y)$ to obtain logits $Z^T, Z^S$, features $\{F_i^T, F_i^S\}_{i=1}^K$, and $\mathcal{L}_{GT}^{(t)}$.
    
    \STATE // Feature-Level Stabilization (LayerNorm \& Stage EMA):
    \FOR{stage $i = 1, \dots, K$}
        \STATE Decouple features to $\hat{F}_i^S$ (Eq.~\ref{eq:layernorm}) and evaluate batch-averaged similarity $S_t^{(i)}$.
        \STATE Compute dynamic weight $\omega_t^{(i)}$ via deviation $B_{t-1}^{(i)} - S_t^{(i)}$, and update baseline $B_t^{(i)}$ (Eq.~\ref{eq:ema}).
    \ENDFOR
    \STATE Aggregate intermediate loss $\mathcal{L}_{Inter}^{(t)} = \sum_{i=1}^K \omega_t^{(i)} \mathcal{L}_{SPOFA}^{(i, t)}$ (Eq.~\ref{eq:kd_inter}).
    
    \STATE  // Gradient-Level Stabilization (Probing \& MEMA):
    \STATE Probing: Compute sample-wise gradients $\nabla_{Z^S} \mathcal{L}_{GT}^{(t,j)}, \nabla_{Z^S} \mathcal{L}_{KD\_Logit}^{(t,j)}$ to get per-sample similarities $\{S_{t, j}\}_{j=1}^N$ (Eq.~\ref{eq:sim}).
    \STATE Evaluate deviation $D_{t, j} \leftarrow \max(0, B_{t-1} - S_{t, j})$ to update dynamic weights $\omega_{t, j}^{out}$ (Eq.~\ref{eq:weight}).
    \STATE Update global baseline $B_t$ (Eq.~\ref{eq:ema}) and compute output loss $\mathcal{L}_{Out}^{(t)}$ (Eq.~\ref{eq:kd_out}).
    
    \STATE // Overall Optimization:
    \STATE Compute $\mathcal{L}_{Total}^{(t)} \leftarrow \lambda_{gt} \mathcal{L}_{GT}^{(t)} + \lambda_{out} \mathcal{L}_{Out}^{(t)} + \lambda_{inter} \mathcal{L}_{Inter}^{(t)}$.
    \STATE Update student parameters: $\theta \leftarrow \theta - \eta \nabla_{\theta} \mathcal{L}_{Total}^{(t)}$.
\ENDFOR
\end{algorithmic}
\end{algorithm}

%实验部分的内容
\section{Experimental Results}
We systematically and comprehensively evaluate the proposed SPOFA framework on two mainstream image classification benchmarks: CIFAR-100 and ImageNet-1K.

% =========================================================
\subsection{Experimental Settings}

\subsubsection{Datasets.}
CIFAR-100~\cite{krizhevsky2009learning} contains 100 categories, each with an image resolution of $32 \times 32$. The training set consists of 50,000 images, and the test set contains 10,000 images. It is a small-scale yet moderately difficult benchmark, widely used to evaluate model generalization and knowledge transfer under limited data scenarios. 
ImageNet-1K~\cite{deng2009imagenet} includes 1,000 generic categories with about 1.2 million training images and 50,000 validation images, covering a wide range of classes with fine-grained granularity. It serves as an authoritative and highly challenging benchmark for visual recognition performance.
Conducting experiments on both datasets allows us to examine SPOFA's behavior in both small- and large-scale scenarios, thereby comprehensively assessing its stability, efficiency, and generality. To ensure consistency with current visual recognition practices, we strictly follow the standard settings used in recent large-scale training and distillation works~\cite{ofa2023neurips,liu2022swinv2}.

\subsubsection{Model Architectures.}
To verify the applicability and robustness of SPOFA across severe heterogeneous gaps, we experiment with multiple teacher-student pairs encompassing three mainstream paradigms:
(1) Convolutional Neural Networks (CNNs): including ResNet, MobileNetV2, and ConvNeXt, representing the evolution from classic to modern CNN designs~\cite{convnext2022}.
(2) Transformer-based Models: such as ViT, DeiT, and Swin Transformer, which have achieved strong performance in vision tasks but differ fundamentally from CNNs in inductive biases and spatial feature scales~\cite{khan2022vitsurvey,liu2022swinv2}.
(3) MLP-based Architectures: such as MLP-Mixer and ResMLP, representing non-convolutional and non-attention spatial paradigms~\cite{khan2022vitsurvey}.
Experiments across these diverse architectures robustly validate whether SPOFA can bridge extreme structural gaps and maintain stable distillation trajectories under disparate inductive biases~\cite{ofa2023neurips,liu2022crosskd,ma2024aligning}.

\subsubsection{Implementation Details.}
All models are implemented using PyTorch and trained under a unified experimental protocol to ensure strictly fair comparisons.
For optimization, we adopt distinct strategies tailored to the student architectures. CNN-based students are trained using the SGD optimizer with a momentum of $0.9$ and a weight decay of $2 \times 10^{-3}$, following a cosine learning rate decay schedule. In contrast, ViT and MLP-based students utilize the AdamW optimizer with a weight decay of $0.05$. A cosine learning rate schedule with a 20-epoch linear warmup is adopted to facilitate stable convergence.
To further regularize training and prevent overfitting in these data-hungry architectures, we employ gradient clipping with a maximum norm of $5.0$ and apply a strong augmentation strategy that combines AutoAugment, Mixup ($\alpha=0.8$), and CutMix ($\alpha=1.0$).
All hyperparameters remain consistent across all competing methods to isolate algorithmic effectiveness. For SPOFA, the momentum coefficient $\alpha$ of the MEMA scaler is set to $0.99$ by default.
Crucially, to ensure a rigorous and fair comparison, our training configurations strictly adhere to the experimental protocols established by the OFA baseline. We maintain identical training durations (300 epochs for all models on CIFAR-100; 100 epochs for CNN students and 300 epochs for others on ImageNet-1K). By eliminating confounding variables such as extended training schedules or specialized tricks, we ensure that the observed performance gains are entirely attributable to the inherent algorithmic advancements of the SPOFA framework.

\subsubsection{Comparison Methods.}
We compare SPOFA against three broad categories of state-of-the-art distillation methods:
(1) Response-based KD: KD~\cite{hinton2015distilling}, DKD~\cite{zhao2022decoupled}, and DIST~\cite{huang2022knowledge}.
(2) Feature-based KD: FitNet~\cite{romero2014fitnets}, RKD~\cite{park2019relational}, and CRD~\cite{tian2019contrastive}.
(3) Heterogeneous KD: The strong baseline OFA~\cite{ofa2023neurips} (NeurIPS 2023) and the recent state-of-the-art method PAT~\cite{lin2025perspective} (ICCV 2025). 
These comparisons comprehensively demonstrate SPOFA's superiority in handling spatial norm discrepancies and temporal gradient conflicts compared to both traditional and specialized heterogeneous distillation approaches.

% =========================================================
\subsection{Main Results on CIFAR-100}
Table~\ref{tab:cifar100} summarizes the Top-1 classification accuracy across 12 highly heterogeneous teacher-student pairs on the CIFAR-100 benchmark. Thanks to the spatial structural decoupling and the temporal MEMA-guided trajectory rectification, SPOFA demonstrates exceptionally robust performance. Notably, SPOFA outperforms the strong heterogeneous baseline OFA~\cite{ofa2023neurips} in 11 out of 12 settings. Furthermore, while maintaining a strictly lightweight and efficient training footprint (introducing virtually zero extra computational overhead), SPOFA achieves strictly superior results to the recent computationally heavy SOTA method PAT~\cite{lin2025perspective} in 9 out of 12 settings. Even in the few cases where PAT marginally leads, SPOFA offers a highly favorable trade-off between distillation performance and training efficiency.

\begin{table*}[!t]
\caption{CIFAR-100 Top-1 Test Results (\%). Results are the average of 3 independent trials. The best values are in \textbf{bold}, and the second best are \underline{underlined}.}
\label{tab:cifar100}
\centering
\resizebox{\textwidth}{!}{
\begin{tabular}{|l|l|c|c|c|c|c|c|c|c|c|c|c|}
\hline
\multirow{2}{*}{\textbf{Teacher}} & \multirow{2}{*}{\textbf{Student}} & \multicolumn{2}{c|}{\textbf{From Scratch}} & \multicolumn{3}{c|}{\textbf{Feature-based}} & \multicolumn{3}{c|}{\textbf{Response-based}} & \multicolumn{3}{c|}{\textbf{Heterogeneous-KD}} \\
\cline{3-13}
 & & T. & S. & FitNet & RKD & CRD & KD & DKD & DIST & OFA & PAT & \textbf{SPOFA}\\
 & & & & \tiny{ICLR'15} & \tiny{ICCV'19} & \tiny{ICLR'20} & \tiny{ArXiv'15} & \tiny{CVPR'22} & \tiny{NeurIPS'22} & \tiny{NeurIPS'23} & \tiny{ICCV'25} & \tiny{Ours} \\
\hline
\multicolumn{13}{|l|}{\textbf{CNN-based students}} \\ \hline
Swin-T & ResNet18 & 89.26 & 74.01 & 78.87 & 74.11 & 77.63 & 78.74 & 80.26 & 77.75 & 80.54 & \underline{81.22} & \textbf{82.39}\\
ViT-S & ResNet18 & 92.04 & 74.01 & 77.71 & 73.72 & 76.60 & 77.26 & 78.10 & 76.49 & 80.15 & \underline{80.11} & \textbf{81.26}\\
Mixer-B/16 & ResNet18 & 87.29 & 74.01 & 77.15 & 73.75 & 76.42 & 77.79 & 78.67 & 76.36 & 79.39 & \underline{80.07} & \textbf{81.89}\\
Swin-T & MobileNetV2 & 89.26 & 73.68 & 74.28 & 69.00 & 79.80 & 74.68 & 71.07 & 72.89 & \underline{80.98} & 78.78 & \textbf{81.25}\\
ViT-S & MobileNetV2 & 92.04 & 73.68 & 73.54 & 68.46 & 78.14 & 72.77 & 69.80 & 72.54 & 78.45 & \underline{78.87} & \textbf{80.01}\\
Mixer-B/16 & MobileNetV2 & 87.29 & 73.68 & 73.78 & 68.95 & 78.15 & 73.33 & 70.20 & 73.26 & \underline{78.78} & 78.62 & \textbf{80.96}\\
\hline
\multicolumn{13}{|l|}{\textbf{ViT-based students}} \\ \hline
ConvNeXt-T & DeiT-T & 88.41 & 68.00 & 60.78 & 69.79 & 65.94 & 72.99 & 74.60 & 73.55 & 75.76 & \textbf{79.59} & \underline{78.91}\\
Mixer-B/16 & DeiT-T & 87.29 & 68.00 & 71.05 & 69.89 & 71.36 & 73.44 & 71.67 & 73.90 & 73.90 & \underline{74.66} & \textbf{77.02}\\
ConvNeXt-T & Swin-P & 88.41 & 72.63 & 24.06 & 71.73 & 67.09 & 76.44 & 76.80 & 76.41 & 78.32 & \underline{80.74} & \textbf{81.17}\\
Mixer-B/16 & Swin-P & 87.29 & 72.63 & 75.20 & 70.82 & 75.93 & 76.39 & 75.85 & 76.65 & \underline{78.93} & 78.44 & \textbf{80.41}\\
\hline
\multicolumn{13}{|l|}{\textbf{MLP-based students}} \\ \hline
ConvNeXt-T & ResMLP-S12 & 88.41 & 66.56 & 45.47 & 65.82 & 63.35 & 72.25 & 73.22 & 71.93 & 81.22 & \textbf{83.50} & \underline{82.32}\\
Swin-T & ResMLP-S12 & 89.26 & 66.56 & 63.12 & 64.66 & 61.72 & 71.89 & 72.82 & 11.05 & \underline{80.63} & \textbf{80.94} & 79.81\\
\hline
\end{tabular}
}
\end{table*}

\noindent\textbf{Analysis of Results.}
The experimental results on CIFAR-100 demonstrate that SPOFA consistently outperforms existing baselines across diverse architectural pairs, validating the efficacy of our spatial decoupling and temporal trajectory rectification strategies.

First, regarding the robustness of CNN students, SPOFA proves exceptionally effective in transferring knowledge from global-attention Transformers or MLPs to local-bias CNNs. This scenario typically suffers from massive feature norm discrepancies. However, SPOFA completely dominates this track (winning 6 out of 6 settings). For instance, when distilling from Swin-T to ResNet18, SPOFA reaches a top-1 accuracy of 82.39\%, significantly outperforming the recent SOTA PAT (81.22\%) by +1.17\% and the strong baseline OFA (80.54\%) by +1.85\%. Similarly, in the Mixer-B/16 $\to$ MobileNetV2 setting, SPOFA outperforms OFA by a striking margin of +2.18\% (80.96\% vs. 78.78\%). These substantial gains confirm that our LayerNorm-based decoupling projector successfully suppresses arbitrary magnitude scaling, strictly decoupling feature norm from direction to break representational inertia. This is visually corroborated by the feature maps in Fig.~\ref{fig:feature_analysis}, where SPOFA demonstrates a sharper, more concentrated semantic alignment on target objects compared to OFA.

Second, for architectures with weak inductive biases (Transformers and MLPs), SPOFA also exhibits highly competitive improvements, preventing optimization instability. Notably, in the distillation from Mixer-B/16 to DeiT-T, SPOFA achieves 77.02\%, surpassing PAT (74.66\%) by a massive +2.36\%. In the ConvNeXt-T $\to$ Swin-P setting, SPOFA continues to outperform PAT by +0.43\% (81.17\% vs. 80.74\%). These architectures are highly susceptible to gradient variance and optimization instability under heterogeneous noise. The results verify that our Momentum-driven EMA (MEMA) scaler establishes a robust historical baseline, effectively penalizing conflicting gradient signals to guarantee monotonic and stable convergence.

Finally, concerning the efficiency-accuracy trade-off, a comparative analysis reveals the definitive superiority of our framework. We acknowledge that PAT demonstrates a marginal performance advantage on 3 specific non-CNN student settings (e.g., ConvNeXt-T $\to$ ResMLP-S12, 83.50\% vs. 82.32\%). However, to bridge the heterogeneous gap, PAT relies on heavy auxiliary attention modules and necessitates dual teacher forwarding, resulting in approximately 3$\times$ the training latency and significantly higher memory footprint. In stark contrast, SPOFA achieves SOTA-level results on 9 out of 12 settings while requiring zero extra parameters during inference and introducing virtually zero computational overhead during training. This highly favorable trade-off makes SPOFA a significantly more practical and elegant solution for real-world, resource-constrained deployments.

\subsection{Distillation Results on ImageNet-1K}
To rigorously evaluate the scalability and robustness of SPOFA under large-scale data regimes, we strategically select four highly representative heterogeneous teacher-student settings on the ImageNet-1K benchmark. Rather than exhaustive enumeration, these specific pairs are deliberately chosen because they encompass the most extreme and challenging architectural gaps—spanning from MLP-to-CNN (e.g., Mixer-B/16 to MobileNetV2) and CNN-to-Transformer (e.g., ConvNeXt-T to DeiT-T) paradigms. By testing the framework against these fundamental shifts in inductive biases, we comprehensively validate its generalized applicability and stability. The results are presented in Table~\ref{tab:imagenet1k}. With the significant increase in data scale and diversity, SPOFA exhibits robust scalability and highly competitive performance.

Specifically, SPOFA attains state-of-the-art performance in three out of four evaluated pairs and secures the second-best position in the remaining one, outperforming the strong OFA baseline in all evaluated settings. For instance, we achieve 72.24\% on MobileNetV2 and 74.46\% on DeiT-T when distilling from Mixer-B/16 and ConvNeXt-T, respectively. Notably, in the difficult Mixer-B/16 $\to$ MobileNetV2 (MLP-to-CNN) setting, SPOFA surpasses both OFA (72.12\%) and the recent SOTA method PAT (72.22\%), suggesting its scalability to large-scale settings on a massive dataset.

Previous feature-based methods (e.g., FitNet, CRD) achieve distinctly inferior performance as they force direct global feature matching, remaining vulnerable to the immense feature norm discrepancies inherent in heterogeneous architectures. While the OFA framework mitigates dimensional mismatch via a unified projection, it still suffers from representational inertia due to unbounded feature magnitudes and memoryless temporal optimization.

In contrast, SPOFA's performance confirms that structurally decoupling feature magnitude via a LayerNorm-based decoupling projector, combined with explicitly rectifying conflicting gradients via the Momentum-driven EMA (MEMA) scaler, ensures highly efficient and stable knowledge transfer. Compared to PAT, which relies on heavy auxiliary attention modules and necessitates dual teacher forwarding, SPOFA achieves superior results in 3 out of 4 settings (e.g., consistently outperforming PAT on both Mixer-B/16 $\to$ MobileNetV2 and ConvNeXt-T $\to$ DeiT-T) using a strictly lightweight pipeline. Even when marginally trailing behind PAT in an isolated case (e.g., Swin-T $\to$ ResMLP-S12), SPOFA offers a vastly superior and practical trade-off between distillation performance and training cost. 

Comparing the results on ImageNet-1K with those on CIFAR-100, we observe that while the relative performance margins are narrower due to the increased complexity of the large-scale dataset, SPOFA exhibits superior stability and consistency. This confirms that our fine-grained distillation mechanism scales effectively to high-complexity classification tasks, providing robust improvements across heterogeneous architectures—including CNNs, Transformers, and MLPs.

\begin{table*}[!t]
\caption{Top-1 accuracy (\%) on ImageNet-1K. The best results are indicated in \textbf{bold}, while the second best are \underline{underlined}.}
\label{tab:imagenet1k}
\centering
\resizebox{\textwidth}{!}{
\begin{tabular}{|l|l|c|c|c|c|c|c|c|c|c|c|c|}
\hline
\multirow{2}{*}{\textbf{Teacher}} & \multirow{2}{*}{\textbf{Student}} & \multicolumn{2}{c|}{\textbf{From Scratch}} & \multicolumn{3}{c|}{\textbf{Feature-based}} & \multicolumn{3}{c|}{\textbf{Response-based}} & \multicolumn{3}{c|}{\textbf{Heterogeneous-KD}} \\
\cline{3-13}
 & & T. & S. & FitNet & RKD & CRD & KD & DKD & DIST & OFA & PAT & \textbf{SPOFA} \\
 & & & & \tiny{ICLR'15} & \tiny{ICCV'19} & \tiny{ICLR'20} & \tiny{ArXiv'15} & \tiny{CVPR'22} & \tiny{NeurIPS'22} & \tiny{NeurIPS'23} & \tiny{ICCV'25} & \tiny{Ours} \\
\hline
\multicolumn{13}{|l|}{\textbf{CNN-based students}} \\ \hline
Swin-T & ResNet18 & 81.38 & 69.75 & 71.18 & 68.89 & 69.09 & 71.14 & 71.10 & 70.91 & \underline{71.85} & 71.54& \textbf{71.92}\\
Mixer-B/16 & MobileNetV2 & 76.62 & 68.87 & 71.59 & 69.86 & 68.89 & 71.92 & 70.93 & 71.74 & 72.12 & \underline{72.22} & \textbf{72.24}\\
\hline
\multicolumn{13}{|l|}{\textbf{ViT-based students}} \\ \hline
ConvNeXt-T & DeiT-T & 82.05 & 72.17 & 70.45 & 71.47 & 69.18 & 74.00 & 73.95 & 74.07 & 74.41 & \underline{74.44} & \textbf{74.46}\\
\hline
\multicolumn{13}{|l|}{\textbf{MLP-based students}} \\ \hline
Swin-T & ResMLP-S12 & 81.38 & 76.65 & 76.48 & 75.10 & 73.40 & 76.67 & 76.99 & 77.25 & 77.31 & \textbf{77.59} & \underline{77.41}\\ \hline
\end{tabular}
}
\end{table*}

\subsection{Efficiency Analysis}
\label{sec:efficiency}
Following the performance comparison, to further validate the efficiency advantages of SPOFA mentioned above, we compare the additional trainable parameters required by different frameworks during the training phase in Table~\ref{tab:params_comparison}.

While the OFA framework maintains a relatively low parameter count, it suffers from rigid optimization and representational inertia. The SOTA method PAT, although effective, introduces massive computational overhead due to its heavy auxiliary attention mechanisms and dual-teacher forwarding. For instance, in the Mixer-B/16 $\to$ ResNet18 setting, PAT requires a staggering 87.82M extra parameters, which is computationally prohibitive for resource-constrained scenarios.

In stark contrast, SPOFA achieves an optimal efficiency-accuracy trade-off. By leveraging a minimalist LayerNorm-based decoupling projector and a highly efficient MEMA scaler, SPOFA introduces a virtually negligible number of parameters (i.e., merely the channel-wise affine scaling parameters for normalization). Consequently, SPOFA maintains a minimal parameter footprint that is practically identical to the lightweight OFA baseline. For instance, in the Swin-T $\to$ ResNet18 setting, SPOFA remains at exactly 1.63M. Even in cases where it slightly exceeds OFA (e.g., Mixer-B/16 $\to$ MobileNetV2), the increment is a marginal +0.01M (2.54M vs. 2.53M), completely preserving the model's lightweight nature.

Crucially, for complex ViT-based students (e.g., ConvNeXt-T $\to$ DeiT-T), SPOFA strictly caps the extra parameters at 3.78M, which is almost 4$\times$ fewer than PAT (14.48M) while achieving superior accuracy. This confirms that SPOFA is not only highly accurate but also strictly minimalist and resource-efficient, effectively bypassing the massive overhead typical of recent heterogeneous KD advancements. Most importantly, it is worth noting that for all evaluated methods, the auxiliary projection branches and dynamic scalers are safely discarded post-training. This results in an identical deployed student size (e.g., 11.69M parameters for a standard ResNet-18) with strictly zero extra inference cost across all approaches. Because SPOFA achieves state-of-the-art accuracy under this identical deployment budget, it establishes a new standard for practical and high-performance heterogeneous distillation.

\begin{table}[t]
\caption{Comparison of extra trainable parameters (in Millions) required by OFA, PAT, and our method (SPOFA) during the training phase. Note that for all evaluated methods, the auxiliary modules are discarded after training, resulting in an identical deployed student size (e.g., 11.69M parameters for ResNet-18) with zero extra inference cost. Bold indicates the best results, and \underline{underlined} indicates the second best.}
\label{tab:params_comparison}
\centering
\resizebox{\linewidth}{!}{
\begin{tabular}{|l|l|c|c|c|}
\hline
\multirow{2}{*}{\textbf{Teacher}} & \multirow{2}{*}{\textbf{Student}} & \multicolumn{3}{c|}{\textbf{Extra Params (M)}} \\
\cline{3-5}
 & & OFA & PAT & \textbf{SPOFA} \\
\hline
\multicolumn{5}{|l|}{\textit{\textbf{CNN-based students}}} \\ \hline
Swin-T     & ResNet18    & \textbf{1.63} & 38.52 & \textbf{1.63}\\
ViT-S      & ResNet18    & \textbf{1.19} & 32.57 & \underline{1.20}\\
Mixer-B/16 & ResNet18    & \textbf{1.63} & 87.82 & \textbf{1.63}\\
Swin-T     & MobileNetV2 & \textbf{2.53} & 32.36 & \underline{2.54}\\
ViT-S      & MobileNetV2 & \textbf{2.53} & 26.41 & \underline{2.54}\\
Mixer-B/16 & MobileNetV2 & \textbf{2.53} & 81.66 & \underline{2.54}\\
\hline
\multicolumn{5}{|l|}{\textit{\textbf{ViT-based students}}} \\ \hline
ConvNeXt-T & DeiT-T & \textbf{3.78}  & 14.48 & \textbf{3.78}\\
Mixer-B/16 & DeiT-T & \textbf{3.78}  & 81.33 & \textbf{3.78}\\
ConvNeXt-T & Swin-P & \textbf{14.86} & 17.77 & \textbf{14.86}\\
Mixer-B/16 & Swin-P & \textbf{14.86} & 84.63 & \textbf{14.86}\\
\hline
\multicolumn{5}{|l|}{\textit{\textbf{MLP-based students}}} \\ \hline
ConvNeXt-T & ResMLP-S12 & \textbf{14.94} & 20.77 & \textbf{14.94}\\
Swin-T     & ResMLP-S12 & \textbf{14.94} & 38.33 & \textbf{14.94}\\
\hline
\end{tabular}
}
\end{table}

\subsection{Ablation Study}
\label{sec:ablation}
To comprehensively evaluate the independent and synergistic contributions of the proposed components, we conducted a step-by-step ablation study. We selected the challenging heterogeneous pair Swin-T $\to$ ResNet18 on CIFAR-100 as the representative setting. This pair distills knowledge from a global-attention Transformer teacher to a local-bias CNN student, representing a scenario where feature geometry misalignment is exceptionally severe, and gradient conflicts are prominent.

As detailed in Table~\ref{tab:component_ablation}, we decompose the SPOFA framework into four progressive configurations to isolate the specific effects of the LayerNorm-based decoupling projector and the Momentum-driven EMA (MEMA) Scaler.

\textbf{Variant 1: Baseline (OFA).} The original OFA framework is employed. It uses a standard non-linear projection (e.g., GELU/Identity) without structural magnitude constraints, and applies a fixed, static weight for the distillation loss. This serves as our lower bound for performance comparison (80.54\%).

\textbf{Variant 2: Baseline + Decoupling.} We intervene in the representational space by replacing the standard projection activation with our LayerNorm-based design. The distillation weight remains fixed. This explicitly decouples feature magnitude from direction, suppressing representational inertia and yielding a direct performance boost to 81.19\% (+0.65\%).

\textbf{Variant 3: Baseline + MEMA.} We revert to the standard projector but intervene in the optimization trajectory by introducing the Momentum-driven EMA (MEMA) Scaler. This replaces the fixed loss weight with a momentum-based dynamic scaling that actively penalizes conflicting gradients. This trajectory rectification alone improves the accuracy to 81.80\% (+1.26\%).

\textbf{Variant 4: SPOFA (Full).} This is our complete proposal, seamlessly integrating both geometry decoupling (LayerNorm) and gradient rectification (MEMA). The combination successfully bridges both structural and optimization gaps, achieving the peak performance of 82.39\% (+1.85\% over the baseline).

\begin{table}[t]
\centering
\caption{ABLATION STUDY ON CIFAR-100 WITH SWIN-T - RESNET18 AS THE TEACHER-STUDENT PAIR.}
\label{tab:component_ablation}
\begin{tabular}{|l|c|c|c|c|}
\hline 
\multirow{2}{*}{\textbf{Method Variant}} & \multicolumn{2}{c|}{\textbf{Components}} & \multirow{2}{*}{\textbf{Acc (\%)}} & \multirow{2}{*}{\textbf{Gain}} \\
\cline{2-3} 
& \textbf{LayerNorm} & \textbf{MEMA} & & \\
\hline
Baseline (OFA) & $\times$ & $\times$ & 80.54 & 0.00 \\ 
\hline
+ Decoupling & \checkmark & $\times$ & 81.19 & +0.65 \\
\hline
+ MEMA & $\times$ & \checkmark & 81.80 & +1.26 \\
\hline
\textbf{SPOFA (Full)} & \textbf{\checkmark} & \textbf{\checkmark} & \textbf{82.39} & \textbf{+1.85} \\
\hline
\end{tabular}
\end{table}

To further validate our specific architectural choice for the decoupling projector, we comprehensively compare LayerNorm (LN) against other mainstream normalization techniques, including BatchNorm (BN), GroupNorm (GN), and a baseline without normalization (Identity). For a fair comparison, to strictly isolate the effect of spatial decoupling, all evaluated variants are jointly equipped with the temporal MEMA scaler and evaluated on the highly heterogeneous Swin-T $\to$ ResNet18 setting.

As reported in Table \ref{tab:norm_ablation}, omitting normalization (Identity) yields an accuracy of 81.80\%. When applying the standard BatchNorm, the performance marginally improves to 81.87\% (a mere +0.07\% gain). This negligible improvement highlights a fundamental limitation of BN in heterogeneous distillation scenarios: because heterogeneous feature distributions exhibit massive variance across different samples, forcing batch-wise statistical calculations disrupts the intrinsic manifold of individual features.

Replacing BN with GroupNorm mitigates the cross-sample contamination issue by performing intra-sample normalization, yielding a better accuracy of 82.13\% (+0.33\% gain over the baseline). However, GN still yields suboptimal results compared to LN. GN restricts holistic feature alignment by imposing artificial boundaries (groups) across the channel dimension. This heuristic grouping disrupts the global semantic topologies that are inherently coupled across all channels in Transformer representations.

In stark contrast, LayerNorm normalizes activations across the entire channel dimension for each spatial token independently \cite{li2025unseen}. This design perfectly addresses the severe norm disparities that induce representational inertia in heterogeneous feature spaces \cite{li2025unseen}. By strictly bounding the feature magnitude of individual tokens and preserving the global channel-wise context, LN yields a significant performance leap to 82.39\%, confirming its absolute necessity for optimal spatial decoupling.

\begin{table}[h]
\centering
\caption{Impact of different normalization strategies in the spatial projector (Swin-T $\to$ ResNet18). To strictly isolate the effect of spatial decoupling, all evaluated variants are jointly equipped with the temporal MEMA scaler.}
\label{tab:norm_ablation}
\begin{tabular}{|l|c|c|}
\hline
\textbf{Variant Configuration} & \textbf{Top-1 Acc (\%)} & \textbf{$\Delta$ vs. None} \\
\hline
None (Identity) & 81.80 & 0.00 \\
\hline
BatchNorm (BN)   & 81.87 & +0.07 \\
\hline
GroupNorm (GN) & 82.13 & +0.33 \\
\hline
LayerNorm (LN)  & \textbf{82.39} & \textbf{+0.59} \\
\hline
\end{tabular}
\vspace{-15pt}
\end{table}

%敏感性分析实验
\subsection{Hyperparameter Sensitivity Analysis}
\label{sec:sensitivity}

To further demonstrate the robustness of the SPOFA framework, we conduct a sensitivity analysis on the two critical hyperparameters introduced in the MEMA scaler: the momentum coefficient $\alpha$ and the penalty decay factor $\lambda$. We evaluate their impact on the highly heterogeneous Swin-T $\to$ ResNet18 setting on CIFAR-100.

Impact of the Momentum Coefficient ($\alpha$). The momentum $\alpha$ dictates the memory depth of the historical baseline. As shown in Figure \ref{fig:sensitivity} (Left), we vary $\alpha \in \{0.85, 0.90, 0.95, 0.99\}$. While a slight performance dip is observed at $\alpha=0.95$, the system generally exhibits strong robustness against this temporal parameter. Setting $\alpha = 0.99$ yields the optimal accuracy (82.39\%), indicating that a relatively long and stable historical memory is beneficial for tracking true gradient trajectories rather than reacting to instantaneous mini-batch noise. Notably, SPOFA maintains high performance ($>81.9\%$) across all tested values, proving it does not rely on heavily fine-tuned ``magic numbers''.

Impact of the Penalty Decay Factor ($\lambda$). The factor $\lambda$ controls the sensitivity of the weight penalty when a gradient conflict is detected. We uniformly test $\lambda \in \{0.05, 0.1, 0.2, 0.3, 0.4\}$. As illustrated in Figure \ref{fig:sensitivity} (Right), the performance exhibits a clear inverted-U trend. A moderate penalty ($\lambda = 0.2$) yields the highest accuracy (82.39\%). A highly aggressive penalty (e.g., $\lambda = 0.4$) over-suppresses the distillation signals even during minor expected fluctuations, while a weak penalty ($\lambda = 0.05$) fails to adequately rectify divergent gradients. Nonetheless, the overall variance remains minimal and consistently outperforms the baseline, further confirming the inherent stability of our momentum-guided rectification mechanism.

\begin{figure}[!t]
    \centering
    \includegraphics[width=\linewidth]{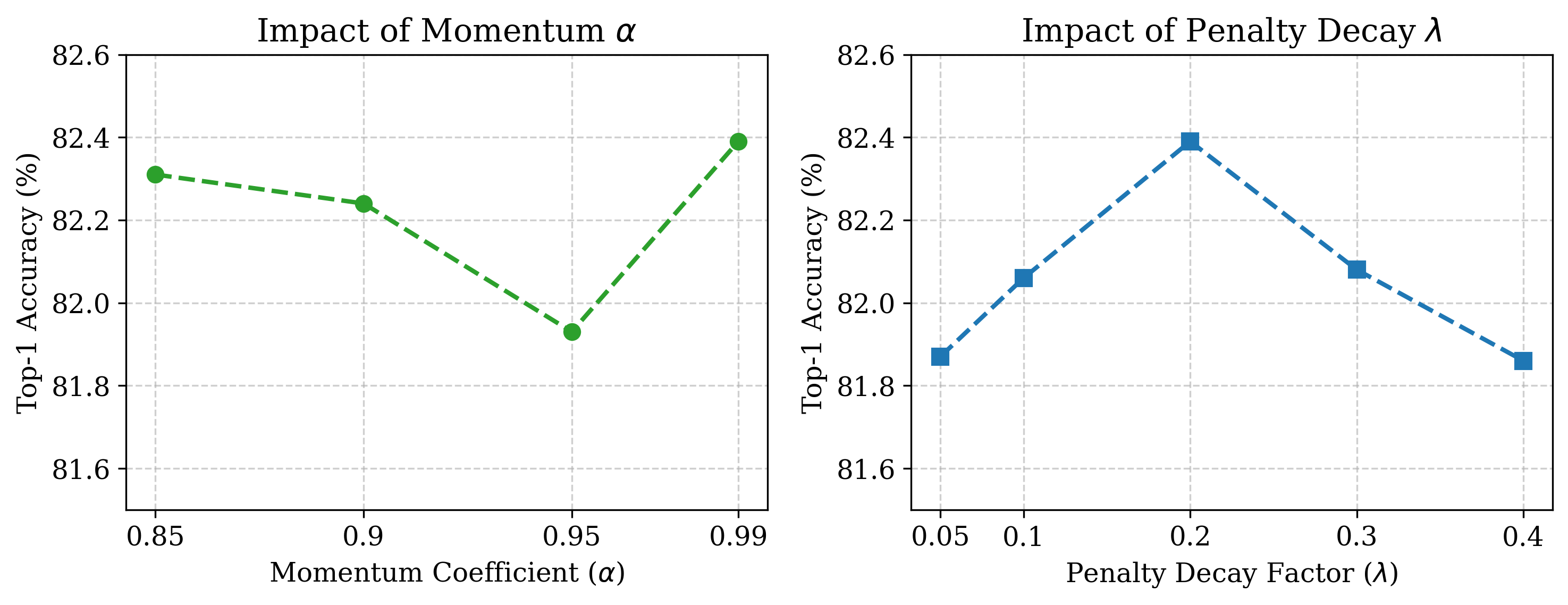} 
    \caption{Hyperparameter sensitivity analysis of the MEMA scaler on CIFAR-100 (Swin-T $\to$ ResNet18). We ablate the momentum coefficient $\alpha$ (Left) and the penalty decay factor $\lambda$ (Right). SPOFA demonstrates robust and consistent improvements over the baseline across a wide range of parameter settings.}
    \label{fig:sensitivity}
    \vspace{-10pt}
\end{figure}

%\subsection{Visualization and Convergence Analysis}
%To provide intuitive insights into the inner workings of SPOFA, we conduct a comprehensive case study on a highly representative heterogeneous pair: Swin-T (Teacher) to ResNet18 (Student). This pair encapsulates the most severe architectural clash: bridging the global self-attention mechanism of Transformers with the local receptive fields of CNNs. We systematically dissect the distillation dynamics of this specific pair across both feature and gradient levels. At the feature level (Structural Gap), the global receptive field of Swin-T intrinsically generates unbounded and massive feature magnitudes, whereas ResNet18 expects normalized, localized variance. Directly matching them induces the representational bottlenecks observed in our preliminary analysis. We visualize how SPOFA's decoupling mechanism resolves this structural gap by bounding the norm and facilitating deep semantic alignment. At the gradient level (Optimization Gap), forcing a CNN to optimize towards attention-based representations triggers severe gradient conflicts, as their inductive biases dictate fundamentally different descent directions. We track how the MEMA scaler prevents the resulting optimization instability, ensuring that the CNN student maintains a stable trajectory and learns holistic semantic topologies. 

\subsection{Visualization and Convergence Analysis}

To provide intuitive insights into the inner workings of SPOFA, we conduct a series of visualizations to verify how our method enhances teacher-student alignment and ensures the stability of the optimization trajectory.

1) Improved Semantic Alignment.
We first examine whether SPOFA effectively bridges the semantic gap at the output level. Figure~\ref{fig:logits_similarity} visualizes the distribution of the logits cosine similarity between the teacher and student models on the CIFAR-100 dataset.
Compared to the OFA baseline, SPOFA exhibits a significantly more compact distribution that is clearly shifted towards higher similarity values. 
This qualitative evidence confirms that our dual-intervention strategy successfully overcomes representational inertia. By maintaining a stable optimization path, SPOFA forces the student to achieve a much higher degree of semantic and distributional alignment with the teacher's predictions.

\begin{figure}[!t] 
    \centering
    \includegraphics[width=0.8\linewidth]{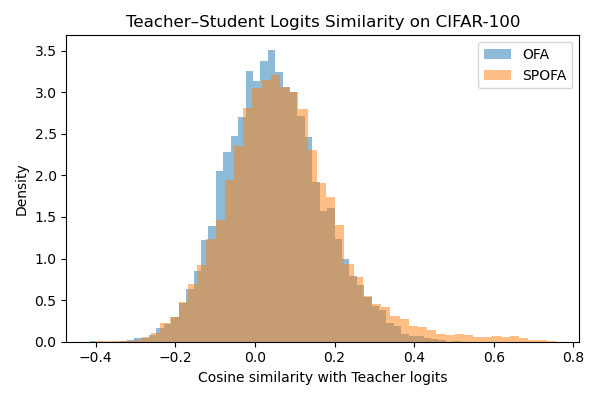} 
    \caption{Comparison of Teacher-Student logits cosine similarity. SPOFA (Orange) demonstrates a distinct shift towards higher similarity compared to the baseline OFA (Blue), indicating superior knowledge transfer. The visualization is conducted on CIFAR-100 with Swin-T - ResNet18 as the teacher-student pair.}
    \label{fig:logits_similarity}
    \vspace{-10pt}
\end{figure}

2) Qualitative Feature Visualization.
To further intuitively demonstrate the synergistic effect of our framework in overcoming both representational inertia and optimization instability, we visualize the feature attention maps of the last convolutional layer of the student network via Grad-CAM. As shown in Fig.~\ref{fig:grad_cam}, due to unmitigated gradient conflicts and heterogeneous norm discrepancies, the optimization trajectory of the baseline OFA-KD frequently drifts. This causes its attention to diverge or collapse into local, sub-optimal discriminative regions (e.g., merely focusing on a local patch of the \textit{Tulip} or erroneously constraining to the lower wall of the \textit{House}). In contrast, by leveraging the MEMA scaler to actively penalize conflicting gradients within the structurally decoupled space provided by the spatial projector, SPOFA-KD maintains a highly stable learning trajectory. Consequently, it generates much sharper and structurally complete heatmaps, comprehensively covering the core semantic regions of the target objects (e.g., extending to the roof of the house and encapsulating the entire beaver's profile). This visually confirms that the stable temporal guidance from MEMA successfully prevents optimization instability, guiding the student to learn the correct holistic semantic topology.

\begin{figure}[!t]
    \centering
    \includegraphics[width=\columnwidth]{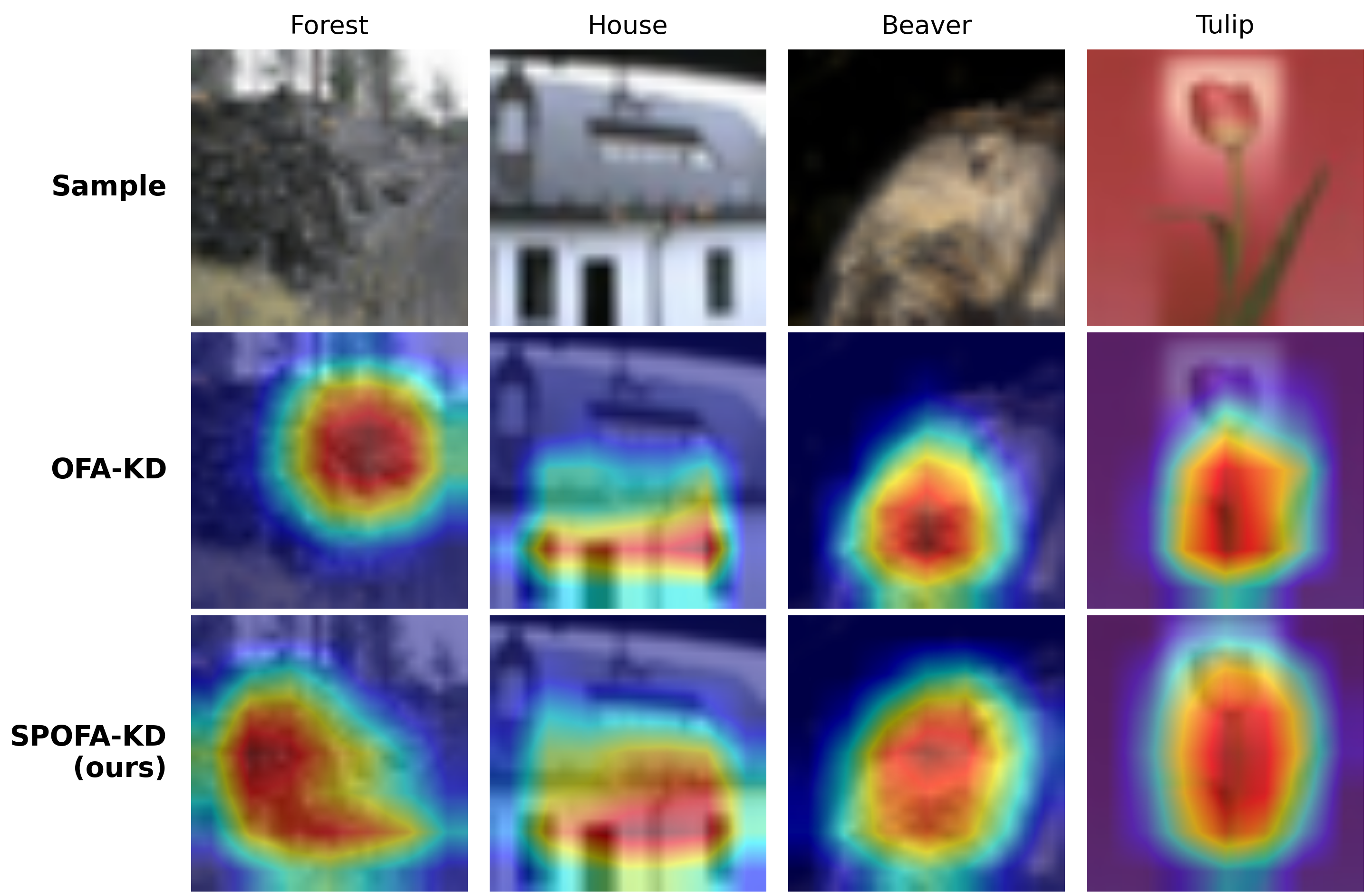} 
    \caption{Qualitative comparison of feature attention maps (Grad-CAM) on the CIFAR-100 test set. Compared to the OFA-KD baseline, SPOFA-KD effectively eliminates representational inertia, leading to a much sharper and more comprehensive semantic focus on the holistic structure of the target objects.}
    \label{fig:grad_cam}
    \vspace{-10pt}
\end{figure}

3) Quantitative Feature Alignment via CKA.
To strictly quantify the representational similarity between the highly heterogeneous teacher and student, we employ Centered Kernel Alignment (CKA), a widely recognized metric for measuring the similarity of feature geometries across different network architectures. We calculate the Linear CKA scores across all four backbone stages using the test set.

\begin{figure}[!t]
    \centering
    \includegraphics[width=0.85\linewidth]{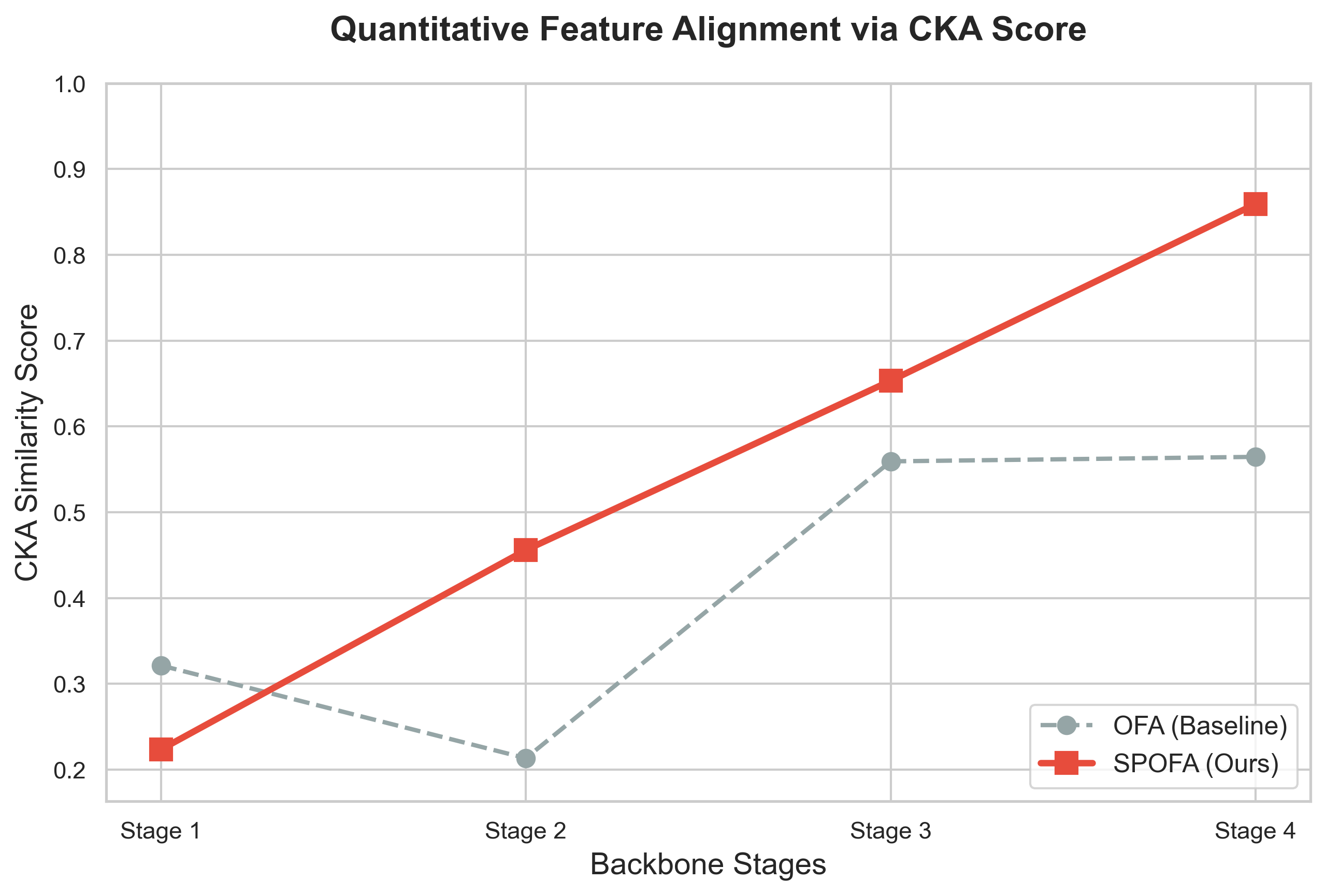}
    \caption{Quantitative analysis of feature alignment using Centered Kernel Alignment (CKA). SPOFA ensures a stable, monotonic increase in similarity, ultimately achieving a significantly higher alignment score (0.86) at the final stage compared to the OFA baseline (0.56). The visualization is conducted on CIFAR-100 with Swin-T - ResNet18 as the teacher-student pair.}
    \label{fig:cka_score}
    \vspace{-10pt}
\end{figure}

As depicted in Figure~\ref{fig:cka_score}, the standard OFA framework (grey dashed line) struggles significantly with the architectural gap. It suffers from a severe alignment collapse at the early-to-mid stages (e.g., dropping to 0.21 at Stage 2) and plateaus at a suboptimal similarity of 0.56 at the final stage. In stark contrast, SPOFA (red solid line) exhibits a healthy, monotonically increasing alignment trajectory. The incorporation of the MEMA scaler effectively prevents the mid-stage collapse. Furthermore, powered by the LayerNorm-based decoupling projector, SPOFA achieves a remarkably high CKA score of 0.86 at Stage 4. This massive gap (+0.30 over OFA) provides rigorous quantitative evidence that SPOFA achieves an unprecedented level of deep semantic alignment.

\begin{figure*}[!t]
    \centering
    \includegraphics[width=\linewidth]{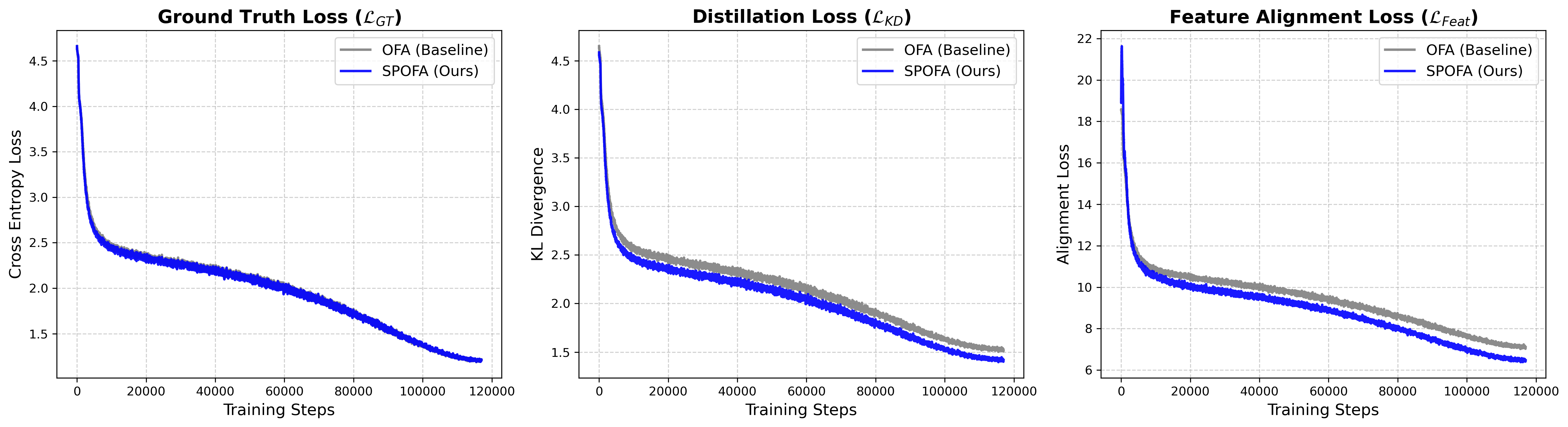}
    \caption{Convergence analysis of training losses over 120,000 steps. SPOFA exhibits accelerated and deeper convergence, particularly in the Distillation Loss ($\mathcal{L}_{KD}$) and Feature Alignment Loss ($\mathcal{L}_{Feat}$), validating the efficacy of LayerNorm and MEMA in stabilizing heterogeneous optimization. The visualization is conducted on CIFAR-100 with Swin-T - ResNet18 as the teacher-student pair.}
    \label{fig:loss_convergence}
\end{figure*}

4) Accelerated and Stabilized Optimization.
To explicitly verify the role of our spatial decoupling (LayerNorm) and temporal rectification (MEMA) mechanisms during training, we track the convergence trajectories of three primary loss components: Ground Truth Loss ($\mathcal{L}_{GT}$), Distillation Loss ($\mathcal{L}_{KD}$), and Feature Alignment Loss ($\mathcal{L}_{Feat}$).

As illustrated in Figure~\ref{fig:loss_convergence}, while both methods achieve comparable convergence on the basic ground truth task (left), SPOFA demonstrates a definitive superiority in the distillation dynamics:
\begin{itemize}
    \item Feature Alignment (Right): The $\mathcal{L}_{Feat}$ curve of SPOFA drops much more rapidly and maintains a strictly lower bound than OFA. This visually validates that replacing standard activations with our LayerNorm-based projector successfully eliminates the magnitude discrepancy, providing an inherently bounded and stable space for feature matching.
    \item Distillation Trajectory (Middle): In the response-based distillation phase ($\mathcal{L}_{KD}$), SPOFA achieves a smoother and deeper convergence. This confirms that the MEMA scaler effectively penalizes conflicting gradients and prevents the "optimization instability" often observed in heterogeneous distillation, ensuring steady descent.
\end{itemize}

These visualizations collectively validate the architectural design of SPOFA. The spatial LayerNorm projector effectively decouples feature magnitude from direction, as evidenced by the rapid convergence of $\mathcal{L}_{Feat}$ (Fig.~\ref{fig:loss_convergence}) and the significantly higher CKA scores across all stages (Fig.~\ref{fig:cka_score}). Meanwhile, the temporal MEMA scaler stabilizes the optimization trajectory, reflected in the smooth descent of $\mathcal{L}_{KD}$ (Fig.~\ref{fig:loss_convergence}) and the sharper, more complete Grad-CAM heatmaps (Fig.~\ref{fig:grad_cam}). Together, these components enable superior semantic alignment at the output level (Fig.~\ref{fig:logits_similarity}).

%结论部分的内容
\section{Conclusion}
In this paper, we presented SPOFA, a robust and highly efficient framework tailored for heterogeneous knowledge distillation. SPOFA addresses the highly coupled feature and optimization misalignments that severely hinder effective knowledge transfer between architecturally divergent models (e.g., CNNs, ViTs, and MLPs). At the core of our framework is a synergistic dual-intervention strategy. At the feature level, SPOFA introduces a minimalist LayerNorm-based decoupling projector that explicitly decouples feature magnitude from direction. This targeted structural alignment effectively suppresses representational inertia and bounds the feature matching space. At the gradient level, to prevent optimization instability caused by heterogeneous architectural biases, we incorporate a Momentum-driven EMA (MEMA) Scaler. By dynamically evaluating gradient consistency and penalizing conflicting optimization steps, MEMA ensures a highly stable and consistent training trajectory. Extensive experiments on CIFAR-100 and ImageNet-1K demonstrate that SPOFA consistently outperforms the strong OFA baseline and surpasses the recent state-of-the-art method PAT (ICCV 2025) across diverse heterogeneous pairs. Crucially, unlike PAT, which relies on heavy auxiliary attention mechanisms and massive computational overhead, SPOFA achieves these state-of-the-art gains with a virtually negligible parameter footprint during training and strictly zero extra inference cost, establishing a new standard for the accuracy-efficiency trade-off. In future work, we plan to extend the principles of SPOFA to more challenging dense prediction tasks, such as object detection and semantic segmentation, and further explore its theoretical potential for resolving norm discrepancies and gradient conflicts in multimodal alignment scenarios.

\section*{Acknowledgments}
This work was supported by the Key Project of Hunan Provincial Department of Education, China (Grant 23A0213), and the Research Fund for Introduced Talents of Central South University of Forestry and Technology (Project Number: 2022YJ006).

%参考文献！

\bibliographystyle{IEEEtran}
\bibliography{reference}

\vfill

\end{document}